\documentclass[12pt,a4paper]{amsart}
\usepackage[T1]{fontenc}
\usepackage[utf8]{inputenc}
\usepackage[a4paper,centering,marginpar=25mm]{geometry}
\usepackage[foot]{amsaddr}
\usepackage[hyperfootnotes=false,colorlinks,breaklinks,urlcolor=blue,linkcolor=blue,citecolor=blue]{hyperref}
\usepackage[nameinlink,capitalise]{cleveref}
\usepackage{graphicx}
\usepackage{subfig}
\usepackage{tikz}
\usepackage[english]{babel}
\newtheorem{thm}{Theorem}[section]
\newtheorem{lemma}[thm]{Lemma}
\newtheorem{prop}[thm]{Proposition}
\newtheorem{corollario}[thm]{Corollary}
\theoremstyle{definition}
\newtheorem{defn}[thm]{Definition}
\theoremstyle{remark}
\newtheorem{esempio}[thm]{Example}
\newtheorem{oss}[thm]{Remark}
\numberwithin{equation}{section}

\def\Rset{\mathbb{R}}

\def\Nset{\mathbb{N}}
\newcommand{\E}{\mathbb{E}}
\newcommand{\Pb}{\mathbb{P}}
\newcommand{\cov}{\operatorname{Cov}}

\begin{document}

\title[Expected signature for time series]{Gaussian processes based data augmentation and expected signature for time series classification}
\author[M. Romito]{Marco Romito}
  \address{Dipartimento di Matematica, Universit\`a di Pisa, Largo Bruno Pontecorvo 5, I--56127 Pisa, Italia}
  \email{\href{mailto:marco.romito@unipi.it}{marco.romito@unipi.it}}
\author[F. Triggiano]{Francesco Triggiano}
  \address{Scuola Normale Superiore, Piazza dei Cavalieri, 7, 56126 Pisa, Italia}
  \email{\href{mailto:francesco.triggiano@sns.it}{francesco.triggiano@sns.it}}
\date{October 16, 2023}
\begin{abstract}
  The signature is a fundamental object that describes paths (that is, continuous functions from an interval to a Euclidean space). Likewise, the expected signature	provides a statistical description of the law of stochastic processes. We propose a feature extraction model for time series built upon the expected signature. This is computed through a Gaussian processes based data augmentation. One of the main features is that an optimal feature extraction is learnt through the supervised task that uses the model.
\end{abstract}
\maketitle

\section{Introduction}

Whenever we are dealing with structured data, such as images or time series, we need to deploy some feature extraction mechanisms in order to solve both supervised and unsupervised tasks. There are various well-known time series features extraction approaches, such as \emph{Catch22} \cite{LSKSFJ2019} or  \emph{Tsfresh} \cite{CBN2018}. Recently, the \emph{signature} of a time series has emerged as a universal non-parametric descriptor of a stream of time ordered data \cite{LyoMcl2022,CK2016}. Signature has arisen in the context of rough path theory \cite{LCL2007,FriVic2010,FriHai2014}, a theory initially developed to analyse irregular signals, and to construct solutions of differential equations driven by such irregular signals. The signature has shown to be a powerful tool to capture the peculiarities of path-like data. In particular, it is able to characterize any path up to adding the time component \cite{HL2010} (see also \cref{Prop4}) and a universal approximation theorem holds \cite{Arribas2018} (see \cref{UnivApprox}).

Lately, the signature transform has been used as features extraction mechanism in neural network-like models \cite{KASL2019} and has been integrated in kernel-based models \cite{KO2019,TotObe2019}. Moreover, signature-based models have been applied in various scientific domain, such as anomaly detection \cite{AGTZ2022} and handwritten text recognition \cite{XSJNL2017}.

Here we are interested in analysing problems where the supervised task should be able to capture the values as well as the statistical features of input data. We have in mind problems such as classification of ECG traces, or of the motion of single cells, see for instance \cite{DSSA2023,VFGJ2018}. We believe that, as the signature is able to give a non-parametric description of paths, the expected signature \cite{Ni2012,CheLyo2016,CheObe2022} is the suitable object to extract features from distributions on paths: the expected signature is a well-suited transform whenever a time series is thought as a trajectory of a stochastic process thanks to its capability of identifying the law of various random processes \cite{CheObe2022}.

In this work we propose a new time series classification model that uses the expected signature as feature extraction procedure. Our method is not simply an architecture of two models placed in series, namely the feature extraction through the evaluation of the expected signature, followed by the classification task based on these features. In our architecture the two models interact and our algorithm learns an optimal evaluation of the expected signature through the feedback of the classification task.
The proposed new feature extraction approach can be compared to the convolutional layers when working on images \cite{gu2018}. Indeed, in both cases the features extraction procedure is learnt by the model itself based on the prediction task at stake. So, it should be deployed to solve supervised tasks, such as classification.

The model combines two main ideas. The first is a stochastic data augmentation based on a Gaussian process regression model. The second idea is to capture the relevant features of paths by means of the expected signature, computed over the ensemble obtained in the phase of data augmentation.

While our main achievement is to combine these two ideas, they both appear, in some way but never together, in previous literature. In \cite{LSDBL2020} the expected signature is used in order to solve supervised tasks where any data is a labelled family of time series. In contrast in our model the expected signature is used in order to characterize the ensemble generated from a single time series through the augmentation step. In \cite{MHBBR2020} a Gaussian laws based augmentation is combined with the signature transform. Their model differs from ours in two fundamental aspects. First, they select mean and covariance functions as in a classical Gaussian processes regression model, while our model learns these quantities completely on its own. More details on this are given in \cref{dataAug}. Second, they do not need the expected signature because randomness is ruled out by passing only the posterior mean and/or variance.

In conclusion, in the present work,
  \begin{enumerate}
    \item we show that the expected signature is an effective tool for supervised tasks, for instance classification, involving time series, when one wants to capture the statistical features of the series and exploit them for better accuracy;
    \item we develop a \emph{data augmentation/Gaussian Process regression/computation of expected signature} module that is fully compatible with back-propagation, and thus can be seamlessly integrated in any neural network architecture;
    \item we find that signature normalization, which is a crucial step to ensure that signature fully captures the statistical properties of paths, turns out to be crucial to ensure computational stability in the evaluation and use of signature;
    \item we show that the proposed method is effective in a number of benchmark examples.
  \end{enumerate}

The paper is organised as follows. In \cref{s:architecture} we describe how to use the expected signature feature extraction module in a simple classification task. The module is flexible and can be used in more complex classification tasks, as well as regression problems. We discuss some experimental results in \cref{s:results}. The algorithm is analysed both on synthetic and real world datasets, a description of the datasets we have used is given in \cref{s:ImpDetails}. Finally, in \cref{s:appendix} we outline the theoretical background and prove some additional results.

\section{Model Architecture}\label{s:architecture}

In this section we propose an extended description of the easiest classification model that can be built using our new feature extraction approach and point out various possible architecture variations.

\subsection{Preliminaries}

A time series that describes a phenomenon extended in time and that includes the influence of random components can be thought as a set of values sampled from a trajectory of a stochastic process.
We preliminary introduce a series of ideas and notion that aim at modeling the description of time-extended phenomena with random components and that will help in illustrating our method.

\subsubsection{Signature of paths and processes} \label{2.1.1}

We start with a short introduction of the signature of a path and of the expected signature of a stochastic process. Technical details are given in \cref{s:appendix}. The interest is in paths, that is continuous function, that in general have poor properties of regularity, and so can be thought of functions with strong degrees of oscillations. The signature is a universal object that describes the intrinsic nonlinear nature of the path and the response of systems governed by the path. The signature of a path $(X_t)_{t\in[0,T]}$ is composed by the set of all iterated integrals,
\[
  S(X)^n = 
    \int\dots\int_{0<s_1<s_2<\dots<s_n<T}\,dX_{s_1}\otimes\dots\otimes dX_{s_n},
\]
with the convention that $S(X)^0=1$.
The expected signature of a stochastic process, which is nothing else but the family of expectations of all iterated integrals of the process, seen as a path, is able to characterize, in a large number of interesting cases, the law of the random signature, and in turns the statistical properties of the process. More precisely, the expected signature characterizes the law of the process only if properly normalized. A (tensor) normalization $\lambda$ is simply a function that for a given signature $S=(1,s^1,s^2,\dots)$ returns the new element $(1,\lambda(S)s^1,\lambda(S)^2 s^2,\dots)$. Tensor normalization is fully illustrated in \cref{ExpSig}.

\subsubsection{Gaussian Process regression model} \label{GP}

We briefly introduce the idea of the Gaussian Process regression model. A full description can be found
in \cite{WR2006}. First, recall that a Gaussian process is a family of random variables $(X_t)_{t \in [0,T]}$ such that all finite dimensional time-marginals have a joint Gaussian distribution. The law of a Gaussian process is completely determined by its mean and covariance functions, namely $m(t)=\E[X_t]$ and $R(s,t)=\cov(X_s,X_t)$.

Given a time series $x=(x_{t_i})_{i=1}^N$, we look for a function or a set of functions that might have possibly generated the known data and that can be used for interpolating the series at unknown time instants. Let $m(t)$ and $R(s,t)$ be a mean and a covariance function, the corresponding Gaussian process induces a prior distribution over the set of functions. Roughly speaking, this choice reduces the functions that we are taking into account.

By binding together the data and the prior distribution we obtain a set of possible interpolating functions and their likelihood of having generated the time series. Indeed, the possible values assumed by an interpolating function at unknown time instants $(s_j)_{j=1}^M$ and the likelihood of each possible set of values for that time instants are described by the conditional law, that is the Gaussian law with mean and covariance given respectively by
\[
  \begin{gathered}
    m(\overline{s})
      + R(\overline{s},\overline{t}) R(\overline{t},\overline{t})^{-1}(x-m(\overline{t})),\\
    R(\overline{s},\overline{s})
      - R(\overline{s},\overline{t})R(\overline{t},\overline{t})^{-1}R(\overline{t},\overline{s}),
  \end{gathered}
\]
where $m(\overline{s})=(m(s_1),..,m(s_M))$ and
\[
  R(\overline{s},\overline{t})=
  \begin{bmatrix}
    R(s_1,t_1) &\dots &R(s_1,t_N)\\
    \vdots &\ddots &\vdots \\
    R(s_M,t_1) &\dots &R(s_M,t_N)\\
  \end{bmatrix}.
\]

Typically the structure of the mean and covariance functions is chosen a-priori, for instance the square exponential covariance function $R(s,t)=\sigma \exp(-\frac1{2l^2}{(t-s)^2})$, and it remains only to estimate the parameters, which are $\sigma$ and $l$ in the squared exponential covariance function case.

\subsection{The model architecture}\label{arch}

The architecture of the simplest model deploying the module implementing the optimal evaluation of the expected signature subject to the accuracy of the classification task is made of two main parts. 

The first element is constituted of three layers and produces a sample of $K$ time series evaluated in a set of new times. The new series are sampled through a GP regression model, whose mean and covariance functions are parameters of the architecture.

The second element is made of four layers, takes as input the sample of $K$ new time series and evaluates the (normalised) expected signature.

A graphical representation of the architecture is shown in \cref{fig0}. In the following we describe the elements of the model in full details.

\subsubsection{Data augmentation} \label{dataAug}

We turn to the description of the first element, whose role essentially is to perform data augmentation and sampling of $K$ new time series which are coherent with the initial input.
This element is made of three layers. The input layer receives the time series values and the corresponding time instants, $(x_i,t_i)_{i=1}^N$, and the sequence $(s_i)_{i=1}^M$ of new time instants.
The first hidden layer receives a vector $m$ and a lower triangular matrix $V$, the output of a linear transformation. They should ideally represent respectively the mean and the square root of the covariance of the conditional law of $X_{s_1},\dots,X_{s_M}$ given $X_{t_1}=x_1,\dots,X_{t_N}=x_N$, that is
\[
  (X_{s_1},\dots,X_{s_M})|(X_{t_1}=x_1,\dots,X_{t_N}=x_N)
    \sim N(m,VV^T).
\]
The third layer, finally, produces $K$ series sampled from $N(m,VV^T)$. More precisely, it samples $K$ time series at time instants $(s_i)_{i=1}^M$.

We emphasize that the outcome of this part provides a much richer ensemble of time series than the original input, that nevertheless is coherent with the original time series.

\begin{figure}
\centering
\begin{minipage}{0.98\textwidth}
\begin{tikzpicture}
\draw (0,0) rectangle (0.8,0.7) node[pos=0.5] {$t_1$};
\draw (0.8,0) rectangle (3,0.7) node[pos=0.5] {...};
\draw (3,0) rectangle (3.8,0.7) node[pos=0.5] {$t_N$};
\draw [->] (1.9,0.7) -- (3.5,1.3);
\draw [->] (1.9,0.7) -- (8,1.3);
\draw (4.3,0) rectangle (5.1,0.7) node[pos=0.5] {$s_1$};
\draw (5.1,0) rectangle (7.3,0.7) node[pos=0.5] {...};
\draw (7.3,0) rectangle (8.1,0.7) node[pos=0.5] {$s_M$};
\draw [->] (6.3,0.7) -- (4,1.3);
\draw [->] (6.3,0.7) -- (8.5,1.3);
\draw (8.6,0) rectangle (9.4,0.7)node[pos=0.5] {$x_1$};
\draw (9.4,0) rectangle (11.6,0.7) node[pos=0.5] {...};
\draw (11.6,0) rectangle (12.4,0.7) node[pos=0.5] {$x_N$};
\draw (12.5,-0.1) -- (12.8,-0.1);
\draw (12.8,3.6) -- (12.8,-0.1)node[pos=0.5,sloped,above] {data augmentation};
\draw (12.8,3.6) -- (12.5,3.6);
\draw [->] (10.8,0.7) -- (4.5,1.3);
\draw [->] (10.8,0.7) -- (9,1.3);
\draw (2,1.4) rectangle (2.8,2.1)node[pos=0.5] {$m_1$};
\draw (2.8,1.4) rectangle (5,2.1)node[pos=0.5] {...};
\draw (5,1.4) rectangle (6.1,2.1)node[pos=0.5] {$m_{M'}$\footnotemark[1]};
\draw [->] (4,2.1)--(8.5,2.7);
\draw [->] (8.5,2.1)--(9.5,2.7);
\draw (6.5,1.4) rectangle (7.3, 2.1) node[pos=0.5] {$V_1$};
\draw (7.3,1.4) rectangle (9.5,2.1) node[pos=0.5] {...};
\draw (9.5,1.4) rectangle (10.5,2.1) node[pos=0.5]{$V_{M''}$};
\draw (0,2.8) rectangle (4.5,3.5) node[pos=0.5] {$\epsilon_1,\dots,\epsilon_K \sim N(0,I_{M'})$};
\draw [->](4.5,3.15)-- (5,3.15);
\draw (5.1,2.8) rectangle (12.4,3.5) node[pos=0.5]{$y^1=V\epsilon_1+m, \dots, y^K=V\epsilon_K+m$ \footnotemark[2] };
\draw[->] (8.2,3.5)--(3.3,4.1);
\draw[->] (8.4,3.5)--(6.1,4.2);
\draw[->] (8.6,3.5)--(9.8,4.1);
\draw (0,4.2) rectangle (1.3,4.9) node[pos=0.5]{$x_1,t_1$};
\draw (1.3,4.2) rectangle (2.1,4.9) node[pos=0.5]{...};
\draw (2.1,4.2) rectangle (3.4,4.9) node[pos=0.5]{$y_1^1,s_1$};
\draw (3.4,4.2) rectangle (4.2,4.9) node[pos=0.5]{...};
\draw (4.2,4.2) rectangle (5.7,4.9) node[pos=0.5]{$x_{N},t_{N}$};
\draw [->] (2.8,4.9)--(2.8,5.5);
\draw [->] (6.1,4.8)--(6.1,5.5);
\draw [->] (9.8,4.9)--(9.8,5.5);
\draw (5.9, 4.55) circle[radius=0.01];
\draw (6.3, 4.55) circle[radius=0.01];
\draw (6.1, 4.55) circle[radius=0.01];
\draw (6.5,4.2) rectangle (7.8,4.9) node[pos=0.5]{$x_1,t_1$};
\draw (7.8,4.2) rectangle (8.6,4.9) node[pos=0.5]{...};
\draw (8.6,4.2) rectangle (9.9,4.9) node[pos=0.5]{$y_1^K,s_1$};
\draw (9.9,4.2) rectangle (10.7,4.9) node[pos=0.5]{...};
\draw (10.7,4.2) rectangle (12.4,4.9) node[pos=0.5]{$x_{N},t_{N}$};
\draw (12.5,4.1) -- (12.8,4.1);
\draw (12.8,9.3) -- (12.8,4.1)node[pos=0.5,sloped,above] {expected signature};
\draw (12.8,9.3) -- (12.5,9.3);
\draw (0,5.6) rectangle (5.7,6.3) node[pos=0.5]{$S^1_1,S^1_2,$\hspace{4pt}$\cdots$\hspace{4pt}$,S^1_L$  };
\draw [->] (2.8,6.3)--(2.8,6.9);
\draw [->] (6.1,6.2)--(6.1,6.9);
\draw [->] (9.8,6.3)--(9.8,6.9);
\draw (6.1, 5.95) circle[radius=0.01];
\draw (6.3, 5.95) circle[radius=0.01];
\draw (5.9, 5.95) circle[radius=0.01];
\draw (6.5,5.6) rectangle (12.4,6.3) node[pos=0.5]{$S^K_1,S^K_2,$\hspace{4pt}$\cdots$\hspace{4pt}$,S^K_L$};
\draw (0,7) rectangle (5.7,7.7) node[pos=0.5]{$(\lambda^1)S^1_1,(\lambda^1)^2S^1_2,$\hspace{4pt}$\cdots$\hspace{4pt}$,(\lambda^1)^L S^1_L$};
\draw [->] (3,7.7)--(5.8,8.3);
\draw [->] (6.1,7.6)--(6.1,8.3);
\draw [->] (9.8,7.7)--(6.4,8.3);
\draw (6.1, 7.35) circle[radius=0.01];
\draw (6.3, 7.35) circle[radius=0.01];
\draw (5.9, 7.35) circle[radius=0.01];
\draw (6.5,7) rectangle (12.4,7.7) node[pos=0.5]{$(\lambda^K)S^K_1,(\lambda^K)^2S^K_2,$\hspace{3pt}$\cdots$\hspace{3pt}$,(\lambda^K)^L S^K_L$};
\draw (2.6,8.4) rectangle (10,9.2) node[pos=0.5]{$[\sum_{i=1}^K(\lambda^i)S^i_1]/K,$\hspace{4pt}$\cdots$\hspace{4pt}$,[\sum_{i=1}^K(\lambda^i)^L S^i_L]/K$};
\draw [->] (6.1,9.2)--(6.1,9.8);
\draw (4.6,9.9) rectangle (7.8,10.6) node[pos=0.5]{output};
\end{tikzpicture}
\footnotetext[1]{ If $(x_i)_{i=1}^N$ is a $d$-dimensional time series, then $M'=Md$ and $M''=\frac{dM(dM+1)}{2}$ }
\footnotetext[2]{Here,
\[m=(m_1,..,.m_{M'}),\text{\hspace{2pt}}V=
\begin{bmatrix}
V_{1}        &0             &\dots  &0\\
V_{2}        &V_3           &\dots  &0\\
\vdots       &\vdots        &\ddots &\vdots\\
V_{M''-M'+1} &V_{M''-M'+2 } &\dots  &V_{M''}\\
\end{bmatrix}
.\] }
\end{minipage}
\caption{Graphical representation}
\label{fig0}
\end{figure}

In other words, the first part of the model deploys a data augmentation scheme that makes use of the sampling procedure outlined above in order to obtain more information about the original time series.

This scheme is strongly connected to the GP regression model described above, since they both exploit a Gaussian process based interpolation procedure. We stress again that the main difference between them is how the mean and the covariance are tuned. In a classical GP regression model, the mean and covariance structure are specified a-priori. In our approach the mean and the covariance are fully learnt by the model.
This different approach is both a necessity and an improvement. Indeed, we cannot make use of hand-designed mean and covariance functions since we would need one for each time series. At the same time this different tuning procedure is a strength of our model because we are not introducing any constraint on the Gaussian law we are sampling from.

\subsubsection{Evaluation of the expected signature}

The second part of the model is responsible of the extraction of the relevant features from the K enriched time series, and is made of four layers. These layers estimate the expected signature based on the ensemble provided by the first part of the model. Clearly the estimate becomes more and more reliable as long as the size K of the sample increases. A quantitative version of this statement is given by \cref{KInequality}.

The first layer of this phase receives the K time series and applies a dimensional augmentation by adding the time component. The following layer computes the signature of each time series up to a truncation level $L$. Then, the normalization procedure that allows the expected signature to characterize the law of stochastic processes is applied to each signature. The expected signature is estimated by averaging component-wise. Finally, the expected signature estimate is used by a softmax layer in order to classify the starting time series.

At this stage we can appreciate that the normalization procedure, which is a requirement to ensure that the expected signature would characterize the law at the theoretical level, turns out to be a crucial step also from a computational point of view. Indeed, we will see that a loose normalization can make the training unstable, see \cref{NewHyper} for experimental evidence.

In addition, the normalization procedure can also be interpreted as a time series preprocessing technique. Indeed, $(\lambda S_1,\lambda^2S_2,\ldots,\lambda^L S_L)$ is both the normalized signature of a given time series $z=(z_{t_i})_i$ and the signature of the rescaled time series $\lambda z$. We point out that in the machine learning literature one can find several time series normalization methods. Here, they would not be equally effective, since they are not able to preserve the fine theoretical properties of the expected signature. See \cref{DiffNorm} for further details.

A technical novelty of our work is that the normalization constant $\lambda$ is found using only the truncated signature. In \cref{CorProved} we show that the value $\lambda$ we use is a proper approximation of the theoretical value $\lambda_T$. In particular, we prove that $\lambda$ converges to $\lambda_T$, as the truncation threshold of the signature diverges to infinity, and we find an estimate on the convergence rate.

\subsection{Training procedure and architecture modifications}

One of the main features of the model is that it can be trained by using any classical gradient based optimization scheme (e.g. SGD) since back-propagation can be performed.
Indeed, the sampling layer does not interfere with the gradient computation because we are exploiting a well-known Gaussian laws property (if $X\sim N(0,I)$, then $Y=VX+m\sim N(m,VV^T)$) in order to take samples of $N(m,VV^T)$ by just sampling from a standard Gaussian\footnote{In the machine learning community this trick is also known as 'the reparameterization trick' \cite{KW2013}.}.

The signature layer and normalization procedure are both differentiable thanks to formula \eqref{eq:SigPath} and the gradient computed in \cref{NormDerivative}.

The usage of back-propagation suggests that the proposed model can be easily introduced in more complex and deeper architectures. The easiest possible modification of our model architecture can be obtained by increasing the number of layers in the prediction phase. 

There are other possible changes that can be easily implemented. For example, we can introduce any different signature computation algorithm, such as the log-Signature transform \cite{CK2016}, or any time series transformation. Indeed, we have been using the time augmentation because it has a relevant role in various theoretical results (\cref{Prop4} and \cref{UnivApprox}), but it can be replaced by various transformation. An extended list of possible and useful time series transformation can be found in \cite{MFKL2020}.

Another possible modification can be obtained by introducing a limitation on the square root $V$ of the covariance function of the conditional law, in order to reduce the computational burden. For instance a reasonable modification is to set to zero some sub-diagonals, that is, if $V=(v_{i,j})_{i,j=1}^{M'}$, to set $v_{i,j}=0$ for all $(i,j)$ such that $i<j$ and $i<\alpha$. The parameter $\alpha$ can be interpreted as a control on the correlation time-scale. Indeed, with this choice, $X_{s_l}$ and $X_{s_m}$ are correlated if $|l-m|<\alpha$.

Lastly, we indicate two possible strategies to select the new time instants $(s_i)_{i=1}^M$. The first one considers the middle points of the sub-intervals $[t_i,t_{i+1}]$ for $i=1\dots N-1$ as new time instants. A second choice takes time instants smaller than $t_1$ or/and bigger than $t_N$ together with the middle points. Even if these two possibilities looks quite similar, they produce a substantial difference: all the time series generated using the first strategy have some components of their signature that are shared by all the other time series. For example, they all have the same components of the first level of the signature since these components depend only on the first and last value of each time series and they all have as first and last values the corresponding values of the original time series.
Instead, the second strategy makes all the components of the signature affected by the sampling procedure.

\section{Experimental results}\label{s:results}

In this section we perform some experiments on real and synthetic datasets in order to analyze the effect of the new hyperparameters  and to assess the inference capability of the model described in \cref{arch}.

\subsection{Implementation details}\label{s:ImpDetails}

All models have been trained using the SGD optimizer as implemented by \texttt{Pytorch} \cite{PGML2019}. Signature computations were done using the package \texttt{signatory} \cite{KL2020}. The hyperparameters have been tuned using a grid search strategy and cross-validation as validation procedure. Weighted accuracy has been chosen as validation metric due to the strong unbalance of some dataset. The code implementing the model and the generation of the synthetic datasets used is available on a dedicated \emph{GitHub} page \cite{githubpage}.

\subsubsection{Datasets}\label{s:datasets}

The dataset used are of two different types. We have created three synthetic datasets sampling trajectories of the following stochastic processes (see for instance \cite{RevYor1999} for details),
\begin{itemize}
  \item standard Brownian motion,
  \item fractional Brownian motion,
  \item geometric Brownian motion, namely the solution of
    \[
      dX_t
        = \mu x_t\,dt + \sigma X_t\,dB_t,
    \]
  \item Ornstein-Uhlenbeck process, namely the solution of
    \[
      dX_t
        = \alpha(\gamma - X_t)\,dt + \beta\,dB_t,
    \]
  \item the solution of a \emph{stochastic differential equation} with non-linear coefficients,
    \[
      dX_t
        = \bigl(\sqrt{1+X_t^2}+\frac12 X_t\bigr)\,dt
          + \sqrt{1+X_t^2}\,dB_t
    \]
    \item white noise perturbations of the following smooth function
    \[
      f(t)
        = 6sin^3(4\pi t)cos^2(4\pi t).
    \]
\end{itemize}

In particular, the first two syntethic problems, called \textbf{FBM} and \textbf{OU}, aim at discriminating two different fractional Brownian motions and two different Ornstein-Uhlenbeck processes, respectively. Instead, the third dataset, called \textbf{Bidim}, is composed by bidimensional time series obtained from all the stochastic processes listed above, where first and second component of any time series are trajectories of the same random process. 

The second group of datasets has been collected from the \emph{Time series classification} website \cite{tsc}. We have chosen datasets with not too many observations, in order to reduce the computational burden, while keeping reliable results. At the same time, we have tried to use datasets coming from different topics (Ecg, Sensor, Image, $\dots$). In particular, we have used the following datasets: ECG200 (electrical activity recorded during one heartbeat. Here the classes are normal heartbeat and Myocardial Infarction), Epilepsy (tri-axial accelerometer data of healthy participants performing one of four class activities), PowerCons (household electric power consumption in warm/cold season), FacesUCR (rotationally aligned facial outlines of 14 grad students), Ham (spectrographs of French or Spanish dry-cured hams). A full description of these datasets can be found at \emph{Time series classification} website \cite{tsc}.

\subsubsection{Benchmark models}\label{s:BenchMod}

We have used a series of benchmark models in order to compare deterministic augmentation schemes with the stochastic scheme proposed here.
The benchmark models used are the following ones:
\begin{itemize}
  \item \texttt{NoAug model}: no augmentation scheme,
  \item \texttt{FFT model}: fast Fourier transform,
  \item \texttt{CS model}: cubic spline interpolation,
  \item \texttt{GP model}: Gaussian Process regression model.
\end{itemize}
In particular \texttt{NoAug} model receives as input the signature of each time series and applies the normalization procedure discussed in \cref{ExpSig} and a linear layer with a softmax. Instead, \texttt{FFT model}, \texttt{CS model} and \texttt{GP model} apply the \texttt{NoAug model} after the preprocessing phase of the time series. Indeed, any time series is augmented using Fast Fourier transform, cubic spline interpolation or classical GP regression model, respectively.
In particular, in the \texttt{GP model} each time series is augmented using the posterior mean obtained by a GP regression model assuming that the mean is a constant function and that the covariance is a squared exponential function.

\subsection{Results}

In this section we firstly analyze how the new hyperparameters introduced by the proposed data augmentation scheme can affect model performance and stability. Then, we compare the performance of our model with well-known models in the literature and the benchmark models we have introduced in \cref{s:BenchMod}. The comparison with these last models will indicate that our stochastic augmentation module can strongly improve signature-based models.

We preliminary point out that since our model is intrinsically stochastic, we have estimated its performance by running it multiple times with the full test set, and by averaging the obtained accuracy and weighted accuracy. The variance of the output was estimated in the following way:
\begin{itemize}
  \item The trained model receives $50$ times each time series of the test set producing as output 50 vectors of length equal to the number of possible labels, where each vector is a probability distribution over the set $\{1,...,D\}$, and where $D$ is the number of possible labels.
  \item Every time series in the test set is associated to a $D \times D$ covariance matrix obtained from the $50$ corresponding vectors;
  \item The empirical density of the 2-norm of the covariance matrices is computed.
\end{itemize}

\subsubsection{Hyperparameters tuning and normalization}\label{NewHyper}

Our model raises the problem of analysing the effect of a new set of hyperparameters on the performance of the supervised task. As in all signature based models, the level $L$ of the signature truncation is a hyperparameter. Likewise, the data augmentation phase introduces the number $M$ of new time instants that interpolate the original time series. Our model requires two new hyperparameters: the sample size $K$ required for the statistical estimates on the expected signature, and the shape parameter $C$ for the tensor normalization of the signature, whose role is explained below.
In this section we focus in details on $K$ and $C$, and show that they should be properly chosen in order to achieve competitive results.

We first consider the shape parameter $C$. We recall that the introduction of a normalization procedure allows the expected signature to characterize the law of the corresponding stochastic process, see \cref{ExpSig} for further details. In particular, the normalization takes the signature $S=(1,s^1,s^2,\cdots)$ and produces the vector $\lambda S=(1,\lambda(S)s^1,\lambda(S)^2s^2,\cdots),$ with $\lambda(S)$ that is chosen as the only scalar such that $|\lambda S|$ is equal to $\psi(|S|)$. The function $\psi$ should satisfy various theoretical properties, as stated in \cref{NormalizationBuilt}. A possible $\psi$ function is given by
\[
  \psi(\sqrt{x})=\left\{
\begin{aligned}\
      & x  \hspace{3.3cm}\text{  		 if } x\le C \\
         & C+C^{2}(C^{-1}-x^{-1})\text{ otherwise } \\
         \end{aligned}
         \right..
\]
We wish to emphasize that the actual value of $C$, and thus the tensor normalization, does not play a significant role in the characterization of the law of a stochastic process by means of the normalized expected signature (\cref{MainExp}). In other words, any normalization would fit. On the other hand our results shown below prove that the shape parameter $C$ plays a relevant role from an experimental point of view. Indeed, it actually determines if the deployed normalization is too strict or too loose and in turns, if the model may underperform or show instabilities. Both these cases should be avoided when training the proposed model.
\begin{figure}
    \centering\includegraphics[width=0.5\linewidth]{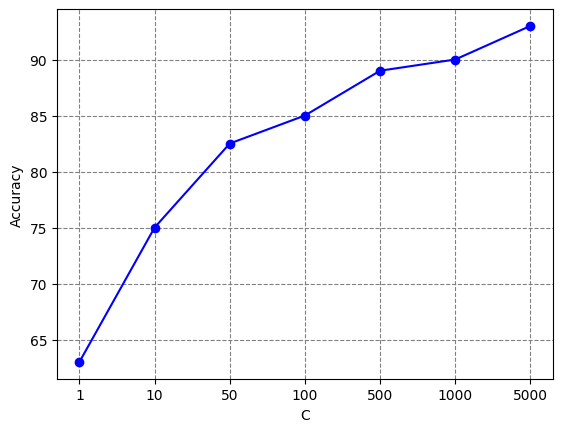}
    \caption{Results obtained working on Bidim dataset.}
    \label{fig1}
\end{figure}

\begin{figure}
    \subfloat[\centering Learning curve of the proposed model.]{{\includegraphics[width=0.5\linewidth]{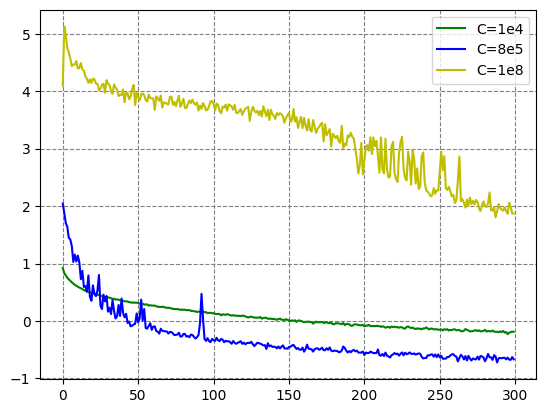} }}
    \subfloat[\centering  Learning curve of the NoAug model.]{{\includegraphics[width=0.5\linewidth]{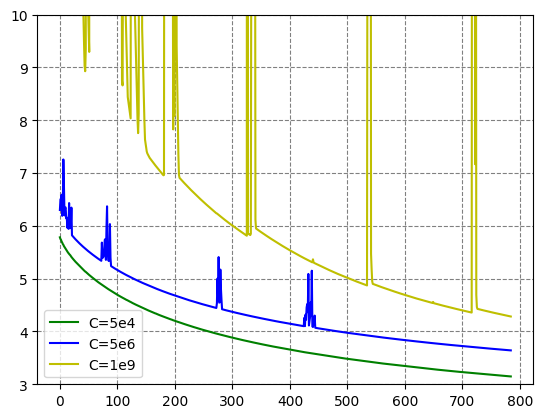} }}%
    \caption{Results obtained working on Bidim dataset.}%
    \label{fig2}%
\end{figure}
\cref{fig1} shows that when $C$ is close to 1, that is in the case the normalization is very rigid, the model can strongly underperform. In contrast, \cref{fig2} shows that a loose normalization can make the training process quite unstable. In particular, the appearance of instabilities even when working with the \texttt{NoAug} model defined in \cref{s:BenchMod}, the easiest model that can be built using the signature as feature extraction mechanism, suggests that the normalization procedure should be deployed and tuned whenever the signature transform is used.

We turn to the analysis of the number $K$ of generated augmented time series. \cref{KInequality} shows that the empirical mean of the $K$ signatures obtained by the $K$ enlarged time series is a good approximation of the expected signature as long as $K$ is sufficiently large.
\cref{fig3} empirically shows that by increasing $K$, that is by getting a better and better approximation of the expected signature, the output of the model becomes more and more stable with respect to the sampling procedure. Indeed, if the model is fed multiple times with the same input, then the variance in the output is small for $K$ big enough.

At last, we highlight that a large value of $K$ can slow down the training phase. Hence, we suggest to look for the smallest $K$ up to a reasonably low variance in the output.
\begin{figure}
    \subfloat[\centering Results obtained working on Bidim dataset.]{{\includegraphics[width=0.5\linewidth]{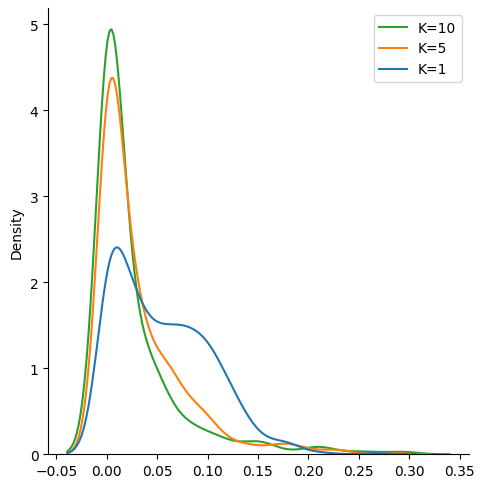} }}
    \subfloat[\centering  Results obtained working on ECG200 dataset.]{{\includegraphics[width=0.5\linewidth]{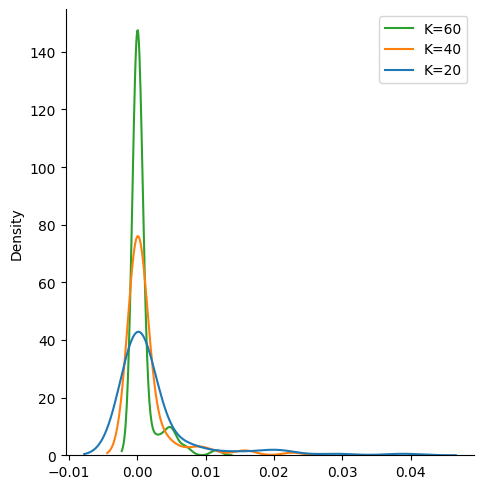} }}%
    \caption{Variance in the output.}%
    \label{fig3}%
\end{figure}

\subsection{Model performance}

In this section we compare the results of our model with some other models on the dataset described in \cref{s:datasets}. In particular, we compare the proposed model with two well-known models, 1NN with DTW \cite{keogh2005exact} and Hive-Cote 2 \cite{middlehurst2021hive}, which are considered state of the art model for time series classification problems, see for instance \cite{MSB2023}, and with the  benchmark models introduced in \cref{s:BenchMod}, in order to show the effectiveness of the stochastic augmentation.

\begin{figure}
    \subfloat[\centering Performance on synthetic data. ]{{\includegraphics[width=0.5\linewidth]{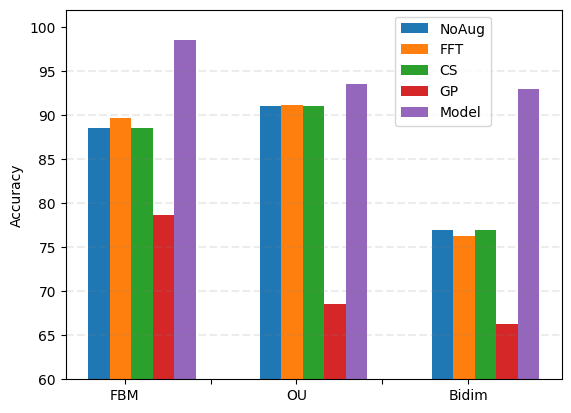} }}
    \subfloat[\centering Performance on real data.]{{\includegraphics[width=0.5\linewidth]{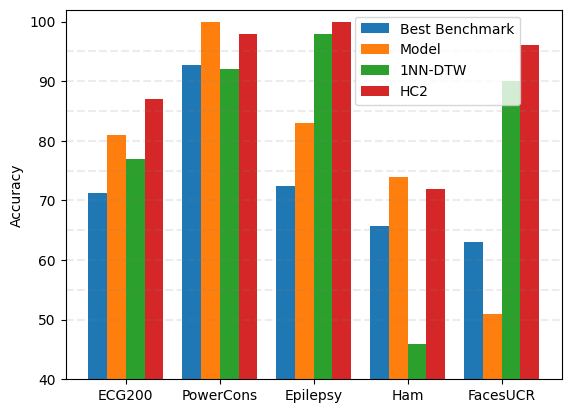} }}%
    \caption{}%
    \label{figRes}%
\end{figure}

\begin{table}[]
\begin{tabular}{|c|c|c|c|c|c|c|c|}
\hline
\textit{Dataset\textbackslash{}Model} & \textit{NoAug} & \textit{FFT} & \textit{CS} & \textit{GP}   & \textit{Model}       & \textit{HC2}  & \textit{1NN-DTW} \\ \hline
\textit{ECG200}                       & 67.5           & 67           & 67          & 70   & 81 (1e-4)            & \textbf{87}   & 77      \\ \hline
\textit{PowerCons}                    & 90             & 90.5         & 90          & 92.7 & \textbf{99.9 (1e-6)} & 98.3          & 92.2    \\ \hline
\textit{Epilepsy}                     & 68.4           & 70.4         & 68          & 72.5 & 83 (4e-4)            & \textbf{100}  & 97.8    \\ \hline
\textit{Ham}                          & 62.6           & 58           & 60          & 65.7 & \textbf{74.2 (2e-4)} & 72.3          & 46.6    \\ \hline
FacesUCR                              & 63             & 63           & 63          & 53   & 51 (1e-5)            & \textbf{96.4} & 90.4    \\ \hline
\end{tabular}
\caption{Accuracy results on real datasets. Accuracy and variance averaged over the test set are reported for our model.}
\end{table}

The comparison with the benchmark models indicates that the stochastic augmentation can strongly increase the inference capability with respect to the model that uses the signature instead of the expected signature, i.e. the \texttt{NoAug} model.

In addition, these results suggest that classical deterministic interpolation schemes are not as effective as the Gaussian processes based augmentation introduced here.
For the sake of completeness  we have shown the results of HC2 and 1NN-DTW. Clearly, we cannot state that the proposed model is statistically comparable to HC2, but the results indicate that the introduced ideas can strongly increase the performance of signature based models.

\subsection*{Acknowledgements}

{M.~R.} acknowledges the partial support of the project PNRR - M4C2 - Investimento 1.3, Partenariato Esteso PE00000013 - \emph{FAIR - Future Artificial Intelligence Research} - Spoke 1 \emph{Human-centered AI}, funded by the European Commission under the NextGeneration EU programme, of the project \emph{Noise in fluid dynamics and related models} funded by the MUR Progetti di Ricerca di Rilevante Interesse Nazionale (PRIN) Bando 2022 - grant 20222YRYSP, and of the project \emph{APRISE - Analysis and Probability in Science} funded by the the University of Pisa, grant PRA\_2022\_85. 

{F.~T.} gratefully acknowledges the Center for High Performance Computing (CHPC) at Scuola Normale Superiore
(SNS) for providing the computational resources used in this work.

\appendix
\section{Theoretical background and some results}\label{s:appendix}

The first of this appendix contains the definition and property of the signature. The second part introduces the expected signature.

\subsection{Signature} \label{Sign}
Almost all the results contained in this section can be found in \cite{LCL2007} and in \cite{FriHai2014}.
Firstly, we need to introduce the tensor space and the truncated tensor space.

\begin{defn}
The tensor product space of $\Rset^d$ is 
\begin{equation*}
    T(\Rset^d)=\{(a^n)_{n \in \Nset}|a^n\in (\Rset^d)^{\otimes n}\},
\end{equation*} the set of the formal series of tensors of $\Rset^d$.
The truncated tensor product space at degree $N$ is the set 
\begin{equation*}
    T(\Rset^d)_L=\{(a^n)_{n \le L}|a^n\in (\Rset^d)^{\otimes n}\}.
\end{equation*}
They are algebras w.r.t the component-wise addition, the component-wise multiplication by scalars and the tensor product 
\begin{equation*}
    a\otimes b=\Bigl(\sum_{i=0}^na^{n-i}\otimes b^{i}\Bigr)_n.
\end{equation*}
\end{defn}

Any tensor component $a^n$ can be represented by its components with respect to the canonical basis of $(\Rset^d)^{\otimes n}$. In other words $a^n$ can be identified by a set of scalar values $(a^{i_1,\dots,i_n})_{i_1,\dots,i_n=1}^d$.
  
Moreover, the set
\[
  T_1(\Rset^d)
    =\Biggl\{a \in T(\Rset^d)| a_0=1, \sqrt{ \sum_{n=0}^{\infty}\sum_{i_1,\dots,i_n=1}^{d}|a^{i_1,\dots,i_n}|^2}< \infty\Biggr\}
\]
is actually a Banach space.

\begin{defn}
Consider $X:[0,T]\to \Rset^d$ a continuous function with bounded variation. The signature of the path $X$ is the sequence of iterated integrals $S(X)_{0,T}=(1,\dots,S(X)^n,\dots )$, where
\begin{equation*}
    S(X)^n=\biggl(\int_0^T\dots\int_0^{u_3}\int_{0}^{u_2}dX_{u_1}^{i_1} dX_{u_2}^{i_2}\dots dX_{u_n}^{i_n}\biggr)_{i_1,\dots,i_n=1\dots d}.
\end{equation*}

\end{defn}

The signature has various properties that makes it well-suited for dealing with path-like data.

\begin{prop}\label{Prop4}
Let $X,Y:[0,T]\rightarrow \Rset^d$ be BV and continuous functions and $\phi:[0,T]\to [0,T]$ be a $C^1$, increasing and surjective function. Then:
\begin{enumerate}
\item $S(X)_{s,t}=S(X)_{s,u}\otimes S(X)_{u,t}$, for all $s<u<t$; 
\item $S(X)=S(X \circ \phi)$;
\item $|S(X)^n|\le \frac1{n!}|X|_{1,[0,T]}^n$, for all $n\in\Nset$, where $|X|_{1,[0,T]}$ is the total variation of $X$; 
\item $S(\overline{X})=S(\overline{Y})$ if and only if $X=Y$, where $\overline{X}_t=(X_t,t)$ (see \cite{HL2010}).
\end{enumerate}
\end{prop}

The third property in \cref{Prop4} shows that by truncating the signature we do not lose a huge amount of information. The fourth property in \cref{Prop4} on the other hand shows that, up to adding the time component, the signature uniquely identifies the corresponding path.

Moreover, it holds a universal approximation theorem.

\begin{thm}{(I.P. Arribas, \cite{Arribas2018})}\label{UnivApprox}
  Let $F:K\to \Rset$ be a continuous function defined over a compact set $K$, composed by continuous and BV functions from $[0,T]$ to $\Rset^d$. Then, for any $\epsilon>0$, there exists a linear map $L$ such that for all $X \in K$,
  \[
    |F(X)-L(S(\overline{X})|\le \epsilon.
  \]
\end{thm}

At last, we indicate how to compute the signature of a time series.

\begin{defn}
    Let $x=(x_{t_i})_{i=0}^N$ be a time series. Its signature is given by the signature of a linear interpolation of $x$.
\end{defn}

A-priori the definition depends on the choice of the linear interpolation (that is the speed at which one traverses the gap between the $x_i$), but \cref{Prop4} ensures that the definition of signature for a stream of data is well-defined and independent from the choice of the linear interpolation. Moreover, \cref{Prop4} allows to easily compute the signature of a time series. Indeed,
\begin{equation}\label{eq:SigPath}
  S(x)
    = \exp(x_1-x_0)\otimes\dots\otimes \exp(x_N-x_{N-1}),
\end{equation}
where we recall that $\exp(a)=\sum_{n=0}^{\infty}\frac{a^{\otimes n}}{n!}$.

\subsection{Expected Signature} \label{ExpSig}
All the results without proof can be found in \cite{CheObe2022}.

\begin{defn}
Consider a stochastic process $(X_t)_{t\in [0,T]}$ such that almost every trajectory is continuous with bounded variation.
The sequence $(\E[S(X)^{i_1,\dots,i_n}])_{i_1,\dots,i_n=1,\dots,d}$ is called the \emph{expected signature} of $X$.
\end{defn}

The expected signature is able to identify the law of the corresponding stochastic process only if it is properly normalized.

\begin{defn}
    A continuous and injective map $\Lambda:T_1(\Rset^d) \rightarrow T_1(\Rset^d)$ is called a \emph{tensor normalization} if there is $\lambda:T_1(\Rset^d) \rightarrow (0, \infty)$ such that:
    \begin{itemize}
       \item $\Lambda(t)=\delta_{\lambda(t)}(t):=(1,\lambda(t)t^1,\lambda(t)^2t^2,\dots)$
         for all $t\in T_1(\Rset^d)$,
       \item $|\Lambda(t)|\le R$ for all $t\in T_1(\Rset^d)$.
\end{itemize}
\end{defn}
Let us show how such a tensor normalization can be built.

\begin{prop}\label{NormalizationBuilt}
    Let $\psi:[1,\infty)\rightarrow [1,\infty)$ be a bounded, injective and $K$-Lipschitz function such that $\psi(1)=1$ and $\sup_{x\geq 1}\frac{\psi(x)}{x^2}\le 1$.
    Given $t \in T_1(\Rset^d)$,  consider as $\lambda(t)$ the only non-negative value such that $|\delta_{\lambda(t)}(t)|^2=\psi(|t|)$.\vspace{5pt}
    Then, the map $\Lambda(t)=\delta_{\lambda(t)}(t)$ is a tensor normalization and there exists a constant $c>0$ such that for all $s,t \in T_1(\Rset^d)$,
    \begin{equation}\label{eq:1}
        |\Lambda(s)-\Lambda(t)|\le c \min(\sqrt{|t-s|},|t-s|).
    \end{equation}
\end{prop}

\begin{esempio}\label{psi}
  The function
  \[
    \psi(\sqrt{x})=
      \begin{cases}
        x & \text{if} x\le C,\\
        C + \frac{C^{1+a}}{a}(C^{-a}-x^{-a}) &\text{otherwise},
      \end{cases}
  \]
  $a>0$ and $C\geq 1$ meets the assumptions of the previous proposition.
\end{esempio}

\begin{corollario} \label{CorProved}
  Let $\psi$ be a function as in the previous proposition, $M \in \Nset^*$, $t \in T_1(\Rset^d)$ and $t_L=(1,t^1,\dots,t^L,0,\dots)$. Then,
  \[
    \lambda_L
      :=\lambda(t_L)
      \rightarrow \lambda(t),
  \]
  as $L\to\infty$.
  
  Moreover, suppose that $t=S(X)$ for some continuous function with bounded variation and consider $r=\min\{j\in\Nset: t^j\neq 0\}$, then for all $L\geq r$,
  \begin{equation}\label{eq:2}
   |\lambda_L-\lambda|
     \leq C\min\Biggl(
       \sqrt[4]{\sum_{j=L+1}^{\infty}\frac1{j!}|X|_{1,[0,T]}^j},
       \sqrt{\sum_{j=L+1}^{\infty}\frac1{j!}|X|_{1,[0,T]}^j}\Biggr)^{\frac{1}{r}}. 
\end{equation}
\end{corollario}

\begin{proof}
  If $r=0$, then $t=(1,0,\dots,0,\dots)=t_L$ and the result is trivial.
  Suppose that $r\neq 0$ and consider $L\geq r$, then
  \begin{equation*}
    |\lambda_L^r-\lambda^r|^2=\frac{|\lambda_L^rt^r-\lambda^rt^r|^2}{|t^r|^2}\le \frac{[\sum_{j=r}^L|\lambda_L^jt^j-\lambda^jt^j|^2+\sum_{j=L+1}^{\infty}|\lambda^jt^j|^2]}{|t^r|^2}=\frac{|\Lambda(t_L)-\Lambda(t)|^2}{|t^r|^2}.
  \end{equation*}
  Therefore, the convergence follows from the continuity of $\Lambda$. The inequality \cref{eq:2} follows from \cref{Prop4} and the inequality \cref{eq:1}.
\end{proof}

Since the normalization procedure is introduced in the proposed model, we need to be able to compute the gradient of $\lambda(t)$.

\begin{corollario}\label{NormDerivative}
    Let $\psi$ be a $C^1$ function that satisfies the assumptions of \cref{NormalizationBuilt}, and consider the corresponding $\lambda(t)$ function.
    
    Given $\overline{t} \in T_1(\Rset^d)_L$ such that $\overline{t}\neq(1,0,\dots,0)$, then there exists an open neighbourhood $U$ of $\overline{t}$ in $ T_1(\Rset^d)_L$ such that $\lambda_{|_U}$ is a $C^1$ function and
    \begin{equation}
      \nabla\lambda(\overline{t})
        =\Biggl(\frac
        {\overline{t}_j^{i_1,\dots,i_j}(\lambda(\overline{t})^{2j}
          - \frac1{2|\overline{t}|}\frac{d}{dx}\psi(|\overline{t}|))}
        {\sum_{k=1}^L k\lambda(\overline{t})^{2k-1}
          \sum_{i_1,\dots,i_k=1}^d|\overline{t}_k^{i_1,\dots,i_k}|^2}
        \Biggr)_{j\in\{1,\dots,L\},i_1,\dots,i_j=1,\dots,d}.
    \end{equation} 
\end{corollario}
\begin{proof}
    Consider the function $F:(0,\infty)\times T_1(\Rset^d)_L \to \Rset$ defined by $F(\lambda,t)=|\delta_{\lambda}(t)|^2-\psi(|t|)$. Its derivatives are given by the following formulas: 
    \[
      \begin{aligned}
        \partial_\lambda F(\lambda,t)
          =\sum_{k=1}^L 2k\lambda^{2k-1}\sum_{i_1,\dots,i_k=1}^d|t_k^{i_1,\dots,i_k}|^2,\\
        \partial_{t_j^{i_1,\dots,i_j}} F(\lambda,t)
          =2t_j^{i_1,\dots,i_j}\Bigl(\lambda^{2j}-\frac1{2|t|}\frac{d}{dx}\psi(|t|)\Bigr).
        \end{aligned}
    \]
    Hence, the result follows directly from the implicit function theorem.
\end{proof}

We can finally state the main property of the expected signature.

\begin{thm} \label{MainExp}
Consider a tensor normalization $\Lambda$ and let $\mu$ and $\nu$ be the laws of $(X_t)_{t\in[0,T]}$ and $(Y_t)_{t\in[0,T]}$, stochastic processes with continuous and BV trajectories. Then, $\mu=\nu$ if and only if $\E[\Lambda( S(\overline{X}))]=\E[\Lambda( S(\overline{Y}))]$
\end{thm} 

Since we estimate the expected normalized signature by averaging the normalized signature of $K$ trajectories, we report a concentration inequality.

\begin{lemma}{(Hoeffding's inequality, \cite{BSB2003})}
 Let $Y_1,\dots,Y_n$ independent random variables such that $Y_i$ takes values in $[a_i,b_i]$ almost surely for all $i\le n$. Then for every $\sigma>0$,
 \begin{equation*}
     \Pb\Bigl[\sum_{i=1}^n(Y_i-E[Y_i])\geq \sigma\Bigr]
       \le \exp\Biggl(-\frac{2\sigma^2}{\sum_{i=1}^n(a_i-b_i)^2}\Biggr).
 \end{equation*}
\end{lemma}

\begin{prop}\label{KInequality}
Let $\Lambda$ be a tensor normalization and $(X_t)_{t \in [0,T]}$ a stochastic process with continuous and BV trajectories.
Consider $Y_1,\dots,Y_K$  \emph{iid} random variables with values in $T_1(\Rset^d)$ such that any $Y_i$ has the same law of the random variable $\Lambda(S(X_t))$. Then for all $\sigma>0$,
\begin{equation*}
    \Pb\Biggl[|\sum_{i=1}^K\frac{Y_i-\E[Y_i]}{K}|\geq \sigma\Biggr]
      \le \exp \Biggl(-\frac{2\sigma^2K}{(2R)^2}\Biggr).
\end{equation*}
\end{prop}

\begin{proof}
  We have
  \[
    \Pb\Biggl[|\sum_{i=1}^K\frac{Y_i-E[Y_i]}{K}|\geq \sigma \Biggr]
      \le \Pb\Biggl[\sum_{i=1}^K\frac{|Y_i-\E[Y_i]|}{K}\geq \sigma \Biggr]
      \le \exp \Biggl(-\frac{2\sigma ^2K^2}{\sum_{i=1}^K(2R)^2}\Biggr),
  \]
  where the last inequality is due to Hoeffding's inequality applied to the random variables $\{\frac1K(|Y_i-\E[Y_i]|)\}_{i=1,\dots,K}$. Indeed, any $\frac1K(|Y_i-\E[Y_i]|)$ takes values in $[0,\frac{2R}{K}]$ since $\Lambda$ is a tensor normalization.   
\end{proof}
At last, we report a concrete example where the normalization is crucial.
\begin{esempio}
Consider two $\Rset^2$-valued stochastic processes 
\[
  (X_t)_{t\in[0,1]}
    =(tN_1,tN_2)_{t\in[0,1]}
    \qquad\text{and}\qquad
  (Y_t)_{t\in[0,1]}=(tM_1,tM_2)_{t\in[0,1]},
\]
where $N=(N_1,N_2)$ and $M=(M_1,M_2)$ have, respectively, density 
\[
  \begin{aligned}
    p(n_1,n_2)
      &= \prod_{i=1}^2 \frac{1}{n_i \sqrt{2\pi}}\exp(-\tfrac12\log^2(n_i)),\\
    q(m_1,m_2)
      &= p(m_1,m_2)\prod_{i=1}^2 (1+\sin(2\pi \log (m_i))).
    \end{aligned}   
\]

These two stochastic processes have the same expected signature.
The proof is an elementary consequence of the following equalities:
\begin{itemize}
  \item $S(X)^{m}=\frac1{m!}((X(1)-X(0))^{\otimes m})$, for all $m \in \Nset$;
  \item $S(X)^{m}=\frac1{m!}((Y(1)-Y(0))^{\otimes m})$, for all $m \in \Nset$;
  \item $\E[(X(1)-X(0))^{\otimes m}]=\E[(Y(1)-Y(0))^{\otimes m}]$, for all $m\in \Nset$.
\end{itemize}
\end{esempio}

\begin{oss} \label{DiffNorm}
We have already highlighted that the normalized signature of a path is, actually, the signature of the path multiplied by a constant. Indeed, $(\lambda S_1,\lambda^2 S_2,\dots,\lambda^L S_L)$ is both the normalized signature of a given path $(Z_t)_t$ and the signature of the rescaled path $(\lambda Z_t)_t$. So, the normalization procedure can be thought as a path preprocessing mechanism. We point out that normalization procedures that are well-known in the machine learning field, such as z-normalization or min-max normalization, are not able to produce a result such as \cref{MainExp}. This can be easily shown by applying them to the previous example.
\end{oss}

\newcommand{\etalchar}[1]{$^{#1}$}
\providecommand{\bysame}{\leavevmode\hbox to3em{\hrulefill}\thinspace}
\providecommand{\MR}{\relax\ifhmode\unskip\space\fi MR }
\providecommand{\MRhref}[2]{%
  \href{http://www.ams.org/mathscinet-getitem?mr=#1}{#2}
}
\providecommand{\href}[2]{#2}

\end{document}